\documentclass[11pt]{article}

\usepackage[margin=1in]{geometry}
\usepackage{amsmath, amssymb, amsthm, mathtools}
\usepackage{stmaryrd}
\usepackage{enumitem}
\usepackage[round]{natbib}
\usepackage{microtype}
\usepackage[hidelinks]{hyperref}

\theoremstyle{plain}
\newtheorem{theorem}{Theorem}
\newtheorem{proposition}{Proposition}
\newtheorem{corollary}{Corollary}
\theoremstyle{definition}
\newtheorem{definition}{Definition}
\newtheorem{example}{Example}
\theoremstyle{remark}

\newcommand{\Tw}{\mathrm{Tw}}
\newcommand{\Comp}{\mathrm{Comp}}
\newcommand{\Int}{\mathrm{Int}}
\newcommand{\Tell}{\mathrm{Tell}}
\newcommand{\supp}{\operatorname{supp}}
\newcommand{\SB}{\mathrm{SB}}
\newcommand{\Nest}{\mathrm{Nest}}

\DeclareMathOperator{\boxright}{\;\Box\!\!\rightarrow\;}

\title{\bf Trust and Its Betrayal under Three Representational Strategies}

\author{Mihnea C. Moldoveanu \and Joel A. C. Baum}

\date{}

\begin{document}
\maketitle
\begin{center}
\vspace{-2.5em}
Rotman School of Management, University of Toronto\\[4pt]
July 16, 2026
\end{center}
\medskip

\begin{abstract}
\noindent Trust is a propositional attitude of a distinctive kind: to trust is to rely on another under conditions where reliance could be disappointed, and the disappointment of trust---betrayal---differs qualitatively from the disappointment of a prediction. We treat trust as a \emph{subjunctive} epistemic state: $A$ trusts $B$'s competence when $A$ believes that \emph{were $P$ true, $B$ would know it}, and $B$'s integrity when $A$ believes that \emph{were $B$ to know $P$, he would disclose it to $A$}. We develop three precise representations of this state---as lexicographic \emph{assumption} \citep{blume1991a, brandenburger2008}, as \emph{ordinal closeness} in a Lewis--Stalnaker sphere system \citep{lewis1973, stalnaker1968}, and as \emph{strong belief} in a conditional probability system \citep{battigalli2002}---and for each we ask whether the Brandenburger--Keisler impossibility on common belief \citep{brandenburger2006} survives when the assumption of rationality is replaced by an assumption of trustworthiness. The three representations agree that every \emph{finite} depth of common trust is realizable while the \emph{completed} common-trust fixed point is the locus of difficulty, but they differ sharply in \emph{how} the difficulty manifests, and---our organizing finding---in how each survives a concrete betrayal. Using three motivating episodes of increasing higher-order structure, and a fully specified six-state model on which all three updates are computed, we show that the same betrayal refutes an agent's \emph{level ordering} under the lexicographic representation, contaminates her \emph{closeness ordering} in proportion to the betrayer's deliberateness under the ordinal representation, and merely \emph{shifts her operative conditioning hypothesis} while leaving her belief structure coherent under the strong-belief representation. The feature that makes each attitude a faithful model of individual trust is, in each case, the feature that governs the fate of common trust and the survivability of betrayal. The paper is intended as a contribution to the epistemic-network (``epinet'') program \citep{moldoveanu2011, moldoveanu2014}, locating trust's two faces---a reflective architecture resident at the nodes and a fixed-point structure resident in the edges---in a single formalism.
\end{abstract}

\section{Introduction}

\subsection{Why trust, and why betrayal}

Trust is among the most consequential and least tractable of the relations that organize economic and social life. It is the precondition of delegation, of testimony, of exchange under incomplete contracts, and of nearly every cooperative undertaking whose full specification would be prohibitively costly \citep{gambetta1988, hardin2002}. Where it is present, parties economize on monitoring, disclosure, and enforcement; where it is absent, these costs return in full, and much that trust makes possible becomes infeasible. It is a commonplace of the organizational and sociological literatures that trust functions as a lubricant for collective action, and of the philosophical literature that it is a distinctive attitude rather than a mere prediction \citep{baier1986, jones1996, holton1994}. What is less often remarked, and what this paper takes as its point of departure, is that these two literatures have not settled on a common account of what trust \emph{is}---and that the disagreement is not idle, because different accounts of the attitude yield different answers to the question that most sharply individuates trust from its neighbors: what counts as its betrayal.

That betrayal is the phenomenon by which trust is known is a recurring theme. Baier's \citeyearpar{baier1986} foundational treatment turns on the observation that trust, unlike mere reliance, can be \emph{betrayed} and not merely disappointed: when a shelf one relied upon to hold a vase gives way, one is disappointed but not betrayed, whereas when a person one trusted fails one, the reactive attitude is of a different and distinctively interpersonal kind. Baier locates the difference in goodwill---to trust is to rely on another's goodwill, and betrayal is the appropriate response to reliance on goodwill that proves absent. Holton \citeyearpar{holton1994}, dissenting from the view that trust is a species of belief, locates it instead in a \emph{participant stance}: to trust is to adopt toward another a readiness to feel betrayal should one's reliance be disappointed, and gratitude should it be upheld. Jones \citeyearpar{jones1996} takes trust to be an affective attitude of optimism about the other's goodwill and competence. Across these otherwise divergent accounts, one structural point is constant: \emph{the possibility of betrayal is the mark of trust's presence}, and the character of the betrayal reveals the content of the trust. Philosophers have accordingly studied betrayal in order to understand trust.

We adopt this methodological orientation but press it in a direction the philosophical literature has not: we make the \emph{representation} of the trusting attitude formally explicit, and we ask how the choice of representation determines what betrayal is. Our claim is that betrayal is not a representation-independent event. What it is for trust to be broken depends on how the trust was held, and agents who hold structurally different attitudes---even toward the same trustee, over the same matter---will be broken differently, will register different events as betrayals, and will face different prospects for repair. Betrayal, on this view, is not a single relation between a truster and a trustee but a family of relations indexed by the form of the truster's commitment.

\subsection{The representation of trust determines what counts as betrayal}

To make this concrete, consider what varies across accounts of trust and what each variation implies for betrayal. The organizational literature, following Mayer, Davis, and Schoorman \citeyearpar{mayer1995}, decomposes trustworthiness into \emph{ability} (competence), \emph{benevolence} (goodwill toward the trustor), and \emph{integrity} (adherence to acceptable principles). A truster who weights ability will count incapacity as betrayal; one who weights benevolence will count indifference as betrayal; one who weights integrity will count principled but harmful conduct as betrayal. The \emph{same} conduct---a disclosure, a silence, an omission---can be a betrayal under one weighting and a forgivable lapse under another. This is familiar. What is not familiar, and what a formal treatment can supply, is that the difference runs deeper than the \emph{content} of what is trusted (competence versus goodwill) to the \emph{form} in which the trust is held---the structure of the truster's epistemic commitment---and that this formal difference is at least as consequential for betrayal as the difference in content.

Compare three trusters who agree entirely on \emph{what} they trust $B$ to do---to know the relevant facts and to disclose them---but differ in \emph{how} they hold that trust. The first holds it as a ranked hypothesis: she treats $B$'s trustworthiness as overwhelmingly her primary expectation while keeping betrayal as a remote but genuine possibility, ordered beneath the expectation. The second holds it as a disposition over the nearest possibilities: she treats the worlds in which $B$ behaves well as the closest ones, so that ``were the question to arise, $B$ would come through'' is true in virtue of what happens across the nearby worlds. The third holds it as a commitment robust to supposition: she believes $B$ trustworthy not merely as things stand but under every hypothesis she can entertain that is consistent with his trustworthiness, relinquishing the belief only under a hypothesis that flatly contradicts it. These three are not weightings of ability against benevolence; they agree on content. They are three \emph{forms} of the same trust, and---this is the burden of the paper---a single betrayal falls upon them differently. What refutes the first is a demotion in a ranking; what refutes the second is a discovery about which worlds were near; what refutes the third is the realization of a hypothesis under which the belief was never asserted. Each is a different way for trust to break, and each leaves a different residue for repair.

That different people may hold structurally different forms of trust is not a marginal possibility to be idealized away but, we suggest, the normal case. A cautious institutional actor who has learned to keep even remote betrayals in view holds trust in the ranked-hypothesis form; a person in an intimate relationship who simply does not entertain the nearby worlds in which the other betrays holds it in the closeness form; a professional who has stress-tested a counterpart across every contingency she can imagine holds it in the robustness form. The three coexist in any population, and even within one person across domains. A theory of betrayal that fixes a single representation therefore describes only one kind of truster; a theory that makes the representation a parameter can say, for each kind, what betrayal is and how it is survived. Supplying that parameterization is our aim.

\subsection{Why three strategies, and what unifies them}

We accordingly develop three precise representations of the same subjunctive trusting attitude and carry each through the same sequence of questions. The common attitude is subjunctive: $A$ trusts $B$'s competence when $A$ holds that \emph{were $P$ true, $B$ would know it}, and $B$'s integrity when $A$ holds that \emph{were $B$ to know $P$, he would disclose it to $A$}. The three representations differ in the formal object that carries the subjunctive: a lexicographic probability system, whose levels rank possibilities \citep{blume1991a, brandenburger2008}; a Lewis--Stalnaker sphere system, whose nested spheres order possibilities by closeness \citep{lewis1973, stalnaker1968}; and a conditional probability system, whose conditional measures encode belief under supposition \citep{battigalli2002}. These are, respectively, the ranked-hypothesis, closeness, and robustness forms of the preceding subsection, now made formally exact.

For each we ask two questions. First, the \emph{interactive} question: trust in its mature form is mutual and iterated---$A$ trusts $B$, trusts that $B$ trusts $A$, and so on---and its natural completion is a \emph{common} attitude of trustworthiness. But common attitudes of exactly this self-referential kind are the subject of an impossibility theorem: \citet{brandenburger2006} show there is no \emph{complete} belief structure hosting the fully iterated hierarchy, and the epistemic characterizations of the relevant solution concepts \citep{brandenburger2008, battigalli2002} presuppose the completeness the theorem denies. We ask, for each representation, whether the impossibility survives when the assumption of \emph{rationality} that figures in these characterizations is replaced by an assumption of \emph{trustworthiness}. Second, the \emph{betrayal} question: for each representation, what does a concrete betrayal do to the truster's structure, and how do the three answers differ? These two questions turn out to be linked. The structural feature of each representation that decides the fate of common trust---the finiteness of a level order, the possible non-well-foundedness of a closeness ordering, the givenness of a conditioning family---is in each case the same feature that decides how betrayal is registered and whether it can be repaired. The feature that makes each attitude a faithful model of individual trust is the feature that governs both the completion of common trust and the survival of its breach.

This is why we engage three strategies rather than settle on one. The point is not that one representation is correct and the others approximations; it is that the three are genuinely different attitudes, applicable to different trusters, and that betrayal means something different under each. Establishing precisely \emph{what} it means under each---and showing that the differences are principled consequences of the representations' formal structure rather than modeling artifacts---is the contribution. We turn now to three episodes that will serve throughout as the touchstone.

\subsection{Three episodes of betrayal}

We begin with the phenomena the theory must explain. Consider three episodes, each an instance of trust disappointed, but of increasing epistemic depth.

\begin{example}[The routed confidence]\label{ex:routed}
$A$ entrusts confidential information $P$ to $B$ under an expectation of discretion. $A$ then discovers that $B$ has routed $P$ to a third party $C$. Moreover, $A$ knows that $B$ knew that $A$ regards $C$ as a party who would use $P$ nefariously. The betrayal is therefore not an innocent lapse: $B$ disclosed $P$ \emph{while knowing} $A$'s valuation of $C$. The violating circumstance is one in which $B$ is fully competent (he knew $P$), correctly models $A$'s interests (he knew $A$'s view of $C$), and betrays notwithstanding.
\end{example}

\begin{example}[The competent silence]\label{ex:silence}
$A$ trusts $B$'s competence: were a material fact $P$ to obtain, $B$ would come to know it and disclose it. A material fact $P$ obtains; $B$ does not disclose it. $A$ later learns that $B$ \emph{did} know $P$ and chose silence, judging that $A$ was better off not knowing. Here integrity fails while competence holds, and the failure is \emph{paternalistic}: $B$'s model of $A$'s interests is intact but overridden. Unlike Example~\ref{ex:routed}, no third party is involved, and the violating circumstance is one of deliberate, well-intentioned non-disclosure.
\end{example}

\begin{example}[The honest incompetent]\label{ex:incompetent}
$A$ trusts $B$ to know $P$ were it true. $P$ obtains; $B$ does not know it, through no bad faith---$B$ simply lacked the capacity or diligence to come to know $P$. When asked, $B$ truthfully reports not knowing. Here \emph{competence} fails while \emph{integrity} is intact: $B$ would have disclosed had he known, but he did not know. The violating circumstance is one of good faith and inadequacy, the mirror image of Example~\ref{ex:silence}.
\end{example}

These three episodes share a surface form---trust disappointed---but differ in the \emph{epistemic location} of the violating circumstance. In Example~\ref{ex:incompetent} the violator is ``shallow'': $B$ is simply not competent, and a world in which $B$ fails to know is, intuitively, a \emph{distant} world for a truster who expected competence. In Examples~\ref{ex:routed} and \ref{ex:silence} the violator is ``deep'': $B$ is competent and correctly models $A$, yet betrays, so the violating world is one in which almost all of the trust conditions hold and only the final disclosure step fails. Our central claim is that the three representations of trust locate deep and shallow violators differently, and consequently survive these betrayals differently. Getting clear on \emph{how} is the work of the paper.

\subsection{Trust as a subjunctive attitude}

What unifies the episodes is that trust concerns what $B$ \emph{would} do across relevant possibilities, including possibilities that do not actually obtain. This is the mark of a counterfactual. On the Lewis--Stalnaker semantics \citep{lewis1973, stalnaker1968}, a counterfactual $\phi \boxright \psi$ is true at a world $w$ when the closest $\phi$-worlds to $w$ are $\psi$-worlds. We accordingly represent the two components of trust as counterfactuals:
\begin{align}
\text{Competence:} &\quad A \text{ believes } \big(P \boxright K_B P\big), \label{eq:comp}\\
\text{Integrity:} &\quad A \text{ believes } \big(K_B P \boxright \Tell_{B,A} P\big), \label{eq:int}
\end{align}
where $K_B P$ is ``$B$ knows $P$'' and $\Tell_{B,A}P$ is ``$B$ discloses $P$ to $A$.'' The two chain to license testimony-based belief transfer: $P \to K_B P \to \Tell_{B,A} P \to (A \text{ comes to believe } P)$.

Two features of the subjunctive reading are worth stating at the outset, because all three representations must respect them. First, \emph{non-inductiveness}: a counterfactual is true in virtue of the closeness structure, not a tally of past instances, so trust does not reduce to an extrapolation from a track record. Second, \emph{brittleness}: if the actual world realizes the antecedent and falsifies the consequent, the counterfactual is \emph{false}, not merely weakened, so a single disconfirming experience violates trust rather than lowering a degree of it. A representation that made betrayal a marginal downward revision of a credence would misrepresent trust; each of our three representations makes betrayal, in its own way, decisive.

\subsection{The interactive setting and the impossibility}

Trust between $A$ and $B$ is interactive and, in its mature form, mutual and iterated: $A$ trusts $B$, $A$ trusts that $B$ trusts $A$, and so on. The natural completion is a \emph{common} assumption of trustworthiness, and this is precisely the object on which \citet{brandenburger2006} proved an impossibility theorem: there is no \emph{complete} belief structure hosting the fully self-referential hierarchy of assumptions about assumptions. Since the epistemic characterizations of solution concepts---\citet{brandenburger2008} for admissibility, \citet{battigalli2002} for forward induction---presuppose completeness, the impossibility bears directly on whether \emph{common trust} can be completed. Our three representations inherit, evade, or condition this impossibility in different ways, and the differences turn out to be the same differences that govern how betrayal is survived.

\subsection{Relation to the epinet program}

This paper contributes to the epistemic-network, or ``epinet,'' program \citep{moldoveanu2011, moldoveanu2014}, which overlays higher-order epistemic states---who knows what, who assumes what about whom---onto the sociometric structure of a network, on the thesis that two networks with identical tie topology can differ radically in epistemic topology. Trust is the paradigmatic epinet relation, and its two faces map onto the two layers of an epinet: a \emph{reflective architecture} resident at the nodes (an agent's own trust dispositions and her access to them) and a \emph{fixed-point structure} resident in the edges (whether mutual trust can be completed across the network). A recurring lesson of what follows is that these two faces are governed by different structural features, and that betrayal acts on both.

\section{Preliminaries and Common Notation}\label{sec:prelim}

We fix notation used by all three representations. Let $A$ be the truster and $B$ the trustee. From $A$'s standpoint the relevant uncertainty is a space $\Omega_B$ of the trustee's strategies and types; symmetrically $\Omega_A$ from $B$'s standpoint. Each state specifies the base facts---whether $P$ holds, whether $B$ knows $P$, whether $B$ discloses $P$ and to whom---through a label in a set $L$. We write
\begin{equation}
\Comp_P = \{\omega : K_B P \text{ at } \omega\}, \qquad
\Int_P = \{\omega : \Tell_{B,A} P \text{ at } \omega\}
\end{equation}
for the competence and integrity events, and $\Tw_B = \Comp_P \cap \Int_P$ for the (depth-one) trustworthiness event. This decomposition is the formal correlate of a distinction the trust literatures have drawn repeatedly in informal terms: it is the ability and integrity components of the \citet{mayer1995} model of trustworthiness, rendered as events in a state space, and it captures the two objects of the optimism that \citet{jones1996} takes to constitute trust---optimism about what the trusted party \emph{can} do and about what she \emph{will} do with what she can. (Benevolence, the third Mayer--Davis--Schoorman component, enters below not as a separate event but through the payoff structure that makes integrity affordable; see Theorem~\ref{prop:closure}.) The \emph{deep-trustworthiness chain} records iterated trust:
\begin{equation}\label{eq:chain}
\Tw^{(1)} = \Tw_B, \qquad
\Tw^{(k+1)} = \Tw_B \cap \{w : w \text{ trusts } A \text{ to depth } k\},
\end{equation}
a descending chain $\Tw^{(1)} \supseteq \Tw^{(2)} \supseteq \cdots$ whose $k$-th term is ``$B$ trustworthy and trusting $A$ trustworthy to depth $k-1$,'' with the reciprocal trust interpreted in the representation-specific sense of each section. Completed common trust is the limit $\Tw^{(\infty)} = \bigcap_k \Tw^{(k)}$.

Throughout, we distinguish the \emph{depth} of trust-iteration (how many levels of ``$A$ trusts that $B$ trusts that\ldots'') from any structure \emph{internal} to a single agent's belief. The three representations differ precisely in that internal structure: a finite order of levels (Section~\ref{sec:lps}), a system of nested spheres (Section~\ref{sec:tss}), and a family of conditional measures (Section~\ref{sec:cps}).

\section{Strategy I: Subjunctive Trust as Lexicographic Assumption}\label{sec:lps}

\subsection{The representation}

The first representation identifies the Lewis--Stalnaker closeness ordering with the level structure of a lexicographic probability system (LPS). An LPS is a finite sequence $\sigma_A = (\mu_0, \dots, \mu_{L-1})$ of probability measures on $\Omega_B$, the levels ranked by primariness: $\mu_0$ the primary hypothesis, $\mu_1$ the secondary, and so on, each infinitely more likely than the next, with acts compared by the lexicographic order on their level-by-level expectations \citep{blume1991a}. Identify the innermost sphere with $\supp(\mu_0)$, the next with $\supp(\mu_1)$, and so on. For an antecedent $P$, the \emph{selected level} is
\begin{equation}\label{eq:sellevel}
\ell(P) = \min\{j : \mu_j(P) > 0\},
\end{equation}
which exists because the LPS has finitely many levels. The trust-counterfactual \eqref{eq:comp} becomes the lexicographic conditional
\begin{equation}
A \text{ subjunctively trusts } B\text{'s competence} \iff \mu_{\ell(P)}\big(K_B P \mid P\big) = 1,
\end{equation}
which is exactly the condition that $A$ \emph{assumes} $B$'s competence in the sense of \citet{brandenburger2008}: at the top segment of levels at which the antecedent is non-null, the consequent holds throughout. Trust is assumption of $\Comp_P$ and $\Int_P$.

An immediate structural consequence will prove decisive. The selected level \eqref{eq:sellevel} is a minimum over finitely many levels, and a minimum over a finite set is always attained. \emph{The lexicographic representation therefore cannot fail the Limit Assumption}; there is never an infinite descending chain of ever-closer levels with no closest. This confines the strategy to the well-founded regime, a fact we exploit in Section~\ref{sec:tss}.

\subsection{Caution and the quarantine of betrayal}

A representation that assigned betrayal probability zero could not represent betrayal: a null event is one the model declares impossible, and a truster for whom deceit is unrepresentable is not trusting but deluded. This is the formal rendering of a point \citet{gambetta1988} made central to the rational-choice treatment of trust: trust operates precisely in the region where defection is genuinely possible---were the probability of defection zero, trust would be superfluous, and were it one, trust would be irrational---so any representation of trust must keep the possibility of betrayal alive within it. The requirement that nothing, betrayal included, be treated as impossible is \emph{caution}, encoded as full support: $\bigcup_j \supp(\mu_j) = \Omega_B$. Caution is compatible with counterfactual trust, by \emph{quarantine}. Consider the integrity counterfactual with antecedent $K_B P$; the betrayal world $b$ (where $B$ knows $P$ but withholds it) satisfies the antecedent. For trust to hold, the selected level $\ell(K_B P)$ must concentrate on disclosure, so $b$ lies outside its support; caution requires $b$ in the support of \emph{some} level. These are reconciled by placing $b$ at a \emph{lower} level: betrayal is reachable, hence non-null and genuinely possible, but distant, hence not governing the counterfactual. The level structure quarantines betrayal beneath the trust-relevant top segment.

\subsection{The level--depth characterization}

The lexicographic setting affords a result with no counterpart in the other two representations, because it concerns level structure specifically: a budget relating the number of levels an agent carries to the depth of common trust she can articulate. Let the shells of the chain \eqref{eq:chain} be
\begin{equation}
W_k = \Tw^{(k)} \setminus \Tw^{(k+1)} \ (k<d), \quad W_d = \Tw^{(d)}, \quad W_\bot = \Omega_B \setminus \Tw^{(1)},
\end{equation}
which partition $\Omega_B$ with the telescoping identity $\Tw^{(k)} = \bigsqcup_{i \ge k} W_i$. Let $n_d$ be the number of distinct nonempty shells among $W_1, \dots, W_d$; under strict descent, $n_d = d$.

\begin{theorem}[Level--depth characterization]\label{thm:leveldepth}
An LPS assumes $\Tw^{(k)}$ for every $k \le d$ simultaneously if and only if it has at least $n_d$ levels; $n_d$ levels suffice. If in addition the LPS is cautious (full support, so $W_\bot$ is non-null), it requires $n_d + 1$ levels when $W_\bot \neq \varnothing$.
\end{theorem}

\begin{proof}
\emph{Sufficiency.} Place each shell at its own level, deepest on top: $\supp(\mu_j) = W_{d-j}$ for $j = 0, \dots, d-1$, and $\supp(\mu_d) = W_\bot$, each $\mu_j$ with full support on its shell. Supports are disjoint (a lexicographic conditional probability system) and their union is $\Omega_B$ (cautious). Fix $k$. The top segment $\bigcup_{j \le d-k} \supp(\mu_j) = W_d \sqcup \cdots \sqcup W_k = \Tw^{(k)}$ by telescoping. For $j \le d-k$, $\supp(\mu_j) = W_{d-j} \subseteq \Tw^{(k)}$, so $\mu_j(\Tw^{(k)}) = 1$ (concentration); the union equals $\Tw^{(k)}$ (exhaustion, with equality); and $\mu_0(\Tw^{(k)}) = 1 > 0$ (non-nullity). So $\Tw^{(k)}$ is assumed, for every $k$ simultaneously.

\emph{Necessity.} Suppose an LPS assumes each $\Tw^{(k)}$ with cutoff $c_k$. Concentration gives $\supp(\mu_j) \subseteq \Tw^{(k)}$ for $j \le c_k$; exhaustion gives $\Tw^{(k)} \subseteq \bigcup_{j \le c_k} \supp(\mu_j)$. Hence $\bigcup_{j \le c_k} \supp(\mu_j) = \Tw^{(k)}$ exactly. Under strict descent the $\Tw^{(k)}$ are $d$ distinct sets, strictly increasing as $k$ decreases; a union of level-prefixes is monotone in the cutoff, so attaining $d$ distinct values requires $d$ distinct cutoffs, hence indices $0, \dots, d-1$, i.e.\ at least $n_d = d$ levels.
\end{proof}

The mechanism is what matters for our purposes. The bound is forced \emph{jointly} by caution (each shell non-null---without it, deep violators go to probability zero and the count collapses) and by the \emph{exhaustion clause} of the assumption operator (each $\Tw^{(k)}$ must be the precise primary region of its top segment, distinguishing assumption from mere probability-one belief---without it, flat belief assumes the whole chain in one level). Neither alone forces the bound; together they force it exactly. It is worth pausing on what the exhaustion clause encodes, because it gives formal content to \citet{holton1994}'s insistence that trust is not a species of belief. Holton's truster can decide to trust without believing, and can trust while retaining a live sense of the ways her reliance could be disappointed; the assumption operator is exactly the attitude with this shape---a commitment to a primary region that neither collapses into probability-one belief (the violators remain non-null) nor reduces to a bet (the primary region is held \emph{as} primary, lexicographically prior to its alternatives). The level--depth budget is thus a cost that attaches specifically to trust in Holton's sense; a truster who merely \emph{believed} would escape it, at the price of no longer holding an attitude that betrayal, as opposed to surprise, could answer to. This is \emph{not} the level structure of iterated admissibility, despite the resemblance: iterated admissibility consumes levels because each round \emph{recomputes} the admissible set relative to the previous round, a fixpoint recursion, whereas the trust chain \eqref{eq:chain} is a \emph{predetermined descending intersection}, and its levels track the shells of a fixed nested family, recorded as such by the exhaustion clause.

\subsection{Inheritance and demarcation}

Because subjunctive trust \emph{is} assumption of trustworthiness under this representation, the principal results transfer by identity: completed common trust is unrealizable at the fixed point by \citet{brandenburger2006}; every finite depth is realizable; and the impossibility is an extensional, inter-agent, structural fact, generated by a diagonal in the composition of the agents' assumption maps and untouched by any intra-agent weakening of an agent's reflective access to her own attitudes. In particular, the \emph{implicitness} of trust---the truster's failure to represent to herself that she trusts---does not let common trust evade the impossibility.

The price of this clean inheritance is a substantive interpretive commitment: the identification reads the counterfactual \emph{doxastically}, as a fact about $A$'s lexicographic belief rather than about mind-independent worldly closeness. For trust, plausibly an attitude of the truster, this is appropriate; but it is exactly the commitment the ordinal strategy refuses. And it has a sharp consequence: since the LPS has finitely many levels and so cannot fail the Limit Assumption, the completed common-trust fixed point is reached by an iteration that closes at a determinate stage and is therefore \emph{bivalent}---realizable or impossible, never anything between. Strategy~I is provably the bivalent strategy.

\section{Strategy II: Subjunctive Trust as Ordinal Closeness}\label{sec:tss}

\subsection{The representation}

The second representation keeps the Lewis--Stalnaker closeness ordering as a primitive rather than reducing it to a measure. A \emph{trust sphere system} (TSS) assigns to each agent $i$ and world $w$ a nest $\$_i(w)$---a family of subsets of $\Omega$ totally ordered by inclusion, with $w$ in the innermost sphere---writing $S^k_i(w)$ for the $k$-th sphere outward. The system is interactive: states carry sphere systems, so $\$_i(w)$ ranks worlds that differ in $j$'s sphere assignment. Trust is the sphere-relative modality
\begin{equation}
\mathsf{T}^{\mathrm{comp}}_A \text{ at } w \iff S^{k^\ast}_A(w) \cap \llbracket P \rrbracket \subseteq \llbracket K_B P \rrbracket, \quad
k^\ast = \min\{k : S^k_A(w) \cap \llbracket P \rrbracket \neq \varnothing\},
\end{equation}
and integrity analogously with antecedent $\llbracket K_B P \rrbracket$. Unlike an LPS, a sphere carries no measure and no normalization; it is a bare ordinal object. This is the price of fidelity to the counterfactual---and, as we will see, it removes the quarantine mechanism that the lexicographic and conditional representations possess.

\subsection{Profligacy as ordinal caution, and the failure of the Limit Assumption}

The ordinal counterpart of caution is \emph{profligacy}: $\bigcup_k S^k_i(w) = \Omega$, so every world is reachable, however distant. Its negation is the existence of unreachable worlds, which the counterfactual treats as strictly impossible. Cautious subjunctive trust is a profligate TSS with trustworthy worlds inner and violators admitted only in outer spheres.

The ordinal setting has a structural degree of freedom the LPS lacks. An LPS has finitely many levels, so a closest level always exists; a sphere system may have \emph{no innermost sphere} meeting a given antecedent---Lewis's \emph{Limit Assumption} can fail, with an infinite descending chain of ever-closer worlds and no closest. This is the crux of the strategy, for it is exactly what permits a third possibility beyond realizable and impossible.

\subsection{The completed fixed point and the third status}

Let $\mathsf{Tr} : \mathcal{P}(\Omega) \to \mathcal{P}(\Omega)$ be the one-step mutual-trust operator. Since the present setting is precisely the one in which the Limit Assumption may fail, we define $\mathsf{Tr}$ by Lewis's assumption-free truth condition rather than by innermost-sphere selection: writing $A_i$ for agent $i$'s trust-relevant antecedent,
\begin{equation}
\mathsf{Tr}(X) = \Big\{\, w : \text{for } i \in \{A,B\},\ \exists S \in \$_i(w)\ \big[\, S \cap A_i \neq \varnothing \ \wedge\ S \cap A_i \subseteq X \,\big] \ \text{or}\ \forall S \in \$_i(w)\ S \cap A_i = \varnothing \,\Big\},
\end{equation}
which coincides with the innermost-sphere reading whenever a closest antecedent-meeting sphere exists, and remains well-defined when none does \citep{lewis1973}. With this, $\Theta = \nu X.\, \mathsf{Tr}(X)$ is the greatest fixed point (completed common trust), computed from above by $X_0 = \Omega$, $X_{n+1} = \mathsf{Tr}(X_n)$, $X_\lambda = \bigcap_{\beta < \lambda} X_\beta$ at limits. Each $X_n$ is the set of worlds at which mutual trust holds to depth $n$.

\begin{definition}[Realization status]\label{def:status}
$\Theta$ is \emph{realizable} if the iteration stabilizes at a nonempty set; \emph{impossible} if it reaches $\varnothing$ at an attained stage; and \emph{liminal} if it is strictly decreasing with every stage nonempty, yet the limit is empty---emptiness approached but never attained at any stage.
\end{definition}

In the liminal case there is no stage at which mutual trust is contradictory; every finite depth is consistent; only the completed infinitary object is empty, and it is empty as a \emph{gap}, a determinate cut with no element at it, as $\sqrt{2}$ is a gap in the rationals. This is what distinguishes liminality from mere undefinedness: the liminal $\Theta$ is fully determinate, the descending tower itself, an object that is the limit of its attained approximants without being among them.

\begin{proposition}[Bivalence under the Limit Assumption]\label{prop:bivalence}
If the TSS is well-founded---every relevant antecedent has a closest sphere---then $\mathsf{Tr}$ preserves the relevant descending intersections, the iteration closes at stage $\omega$, and $\Theta$ is either realizable or impossible. Liminality cannot occur.
\end{proposition}

\begin{proof}
Under the Limit Assumption the innermost-sphere selection is an attained minimum, so $\mathsf{Tr}(\bigcap_n X_n) = \bigcap_n \mathsf{Tr}(X_n) = \bigcap_n X_{n+1} = \bigcap_n X_n$; the stage-$\omega$ intersection is itself a fixed point and is attained as such, hence nonempty (realizable) or empty (impossible, the attained $\varnothing$). No strictly-descending-past-$\omega$ behaviour with empty unattained limit is possible.
\end{proof}

Proposition~\ref{prop:bivalence} locates liminality precisely: it is excluded by well-foundedness, hence excluded in any representation order-isomorphic to a ranking or lexicographic system---which is to say, in Strategies~I and~III. The third status is available \emph{only} when the Limit Assumption fails, and it is therefore invisible to the probabilistic representations and visible only here.

\subsection{The bridge: mutual trust as a predicate lifting}\label{sec:bridge}

The trichotomy of Definition~\ref{def:status} is a fact about the subset tower of $\mathsf{Tr}$ over a fixed state space; the convergence theory invoked in the next subsection concerns the terminal sequence of a functor. These are distinct mathematical registers, and relating them requires more than analogy. The required relation is supplied by the apparatus of coalgebraic modal logic \citep{pattinson2003, jacobs2016}: we exhibit $\mathsf{Tr}$ as a \emph{predicate lifting} of the sphere-system functor, from which the correspondence between the two towers follows by naturality, and the exact sense in which the functorial classification bears on $\Theta$ follows as a proposition rather than a gloss.

An interactive trust sphere system is a coalgebra $\gamma : \Omega \to F(\Omega)$ for $F(X) = L \times \Nest(X)^2$. Because the trust condition involves an antecedent as well as a consequent, the appropriate lifting is binary, in the manner of coalgebraic semantics for conditional logics. For sets $V, U \subseteq X$ define
\begin{equation}\label{eq:lifting}
\tau_X(V, U) = \Big\{ (\ell, N_A, N_B) \in F(X) : \text{for each } N \in \{N_A, N_B\},\ \exists S \in N\ \big[ S \cap V \neq \varnothing \wedge S \cap V \subseteq U \big] \ \text{or}\ \forall S \in N\ S \cap V = \varnothing \Big\}.
\end{equation}

\begin{proposition}[Naturality]\label{prop:naturality}
$\tau$ is a natural transformation $Q \times Q \Rightarrow Q \circ F$ in both arguments, where $Q$ is the contravariant power-set functor: for every $f : X \to Y$ and $V, U \subseteq Y$,
$F(f)^{-1}\big(\tau_Y(V, U)\big) = \tau_X\big(f^{-1}V,\ f^{-1}U\big)$.
Consequently $\mathsf{Tr}(U) = \gamma^{-1}\big(\tau_\Omega(A, U)\big)$, with $A$ the label-definable antecedent: the mutual-trust operator is the predicate lifting of $\tau$ along the coalgebra.
\end{proposition}

\begin{proof}
$F(f)$ acts on nests by direct image, $f[N] = \{f[S] : S \in N\}$; direct images preserve inclusion, so nests map to nests. The required equivalences, for each sphere $S$, are $f[S] \cap V \neq \varnothing \iff S \cap f^{-1}V \neq \varnothing$ and $f[S] \cap V \subseteq U \iff S \cap f^{-1}V \subseteq f^{-1}U$. Both follow from the identity $f[S] \cap V = f[\,S \cap f^{-1}V\,]$---if $v = f(s)$ with $s \in S$ and $v \in V$ then $s \in f^{-1}V$, and conversely---together with the adjunction $f[T] \subseteq U \iff T \subseteq f^{-1}U$. The vacuity disjunct transforms by the same identity, and the conjunction over the two agents is a product of liftings. The displayed identity for $\mathsf{Tr}$ is then the definition of $\mathsf{Tr}$ read through \eqref{eq:lifting}.
\end{proof}

Let $\gamma_n : \Omega \to F^n(1)$ be the canonical cone into the terminal sequence ($\gamma_0$ the unique map, $\gamma_{n+1} = F(\gamma_n) \circ \gamma$), and define the \emph{trust-predicate tower} on the terminal sequence by $P_1$ the label-antecedent stage and $P_{n+1} = \tau_{F^n(1)}(A_n, P_n)$, where $A_n$ is the antecedent at stage $n$; the antecedents cohere across the tower because the connecting maps preserve the label coordinate.

\begin{proposition}[Tower correspondence]\label{prop:towercorr}
For every $n$, $X_n = \gamma_n^{-1}(P_n)$; and at the limit, with $P_\omega = \bigcap_n \pi_n^{-1}(P_n) \subseteq F^\omega(1)$ and $\gamma_\omega$ the canonical map into the limit, $X_\omega = \bigcap_n X_n = \gamma_\omega^{-1}(P_\omega)$.
\end{proposition}

\begin{proof}
Induction on $n$, the base case being the label condition. For the step,
$X_{n+1} = \gamma^{-1}\tau_\Omega(A, X_n) = \gamma^{-1}\tau_\Omega(\gamma_n^{-1}A_n, \gamma_n^{-1}P_n) = \gamma^{-1} F(\gamma_n)^{-1} \tau_{F^n(1)}(A_n, P_n) = \gamma_{n+1}^{-1}(P_{n+1})$,
the middle equality by Proposition~\ref{prop:naturality} at $f = \gamma_n$. The limit case is immediate: preimages commute with intersections, so $\gamma_\omega^{-1}(P_\omega) = \bigcap_n (\pi_n \circ \gamma_\omega)^{-1}(P_n) = \bigcap_n \gamma_n^{-1}(P_n) = \bigcap_n X_n$.
\end{proof}

\begin{proposition}[Refinement, and the coincidence criterion]\label{prop:refinement}
The functorial classification refines, and does not in general determine, the lattice classification of $\Theta$. Precisely: $X_n \neq \varnothing$ implies $P_n \neq \varnothing$, and $\bigcap_n X_n \neq \varnothing$ implies $\varprojlim_n P_n \neq \varnothing$ (the cone of a realizing state is a thread); but the converses fail, since $\varprojlim_n P_n$ may contain threads realized by no state of the given coalgebra. The two classifications coincide exactly when every thread of $\varprojlim_n P_n$ in the closure of the image of $\gamma_\omega$ is realized---in particular whenever $\gamma_\omega$ is surjective onto the thread space.
\end{proposition}

\begin{proof}
The forward implications are Proposition~\ref{prop:towercorr} together with the observation that preimages of nonempty sets along the cone are the stated sets. For failure of the converses it suffices to note that a thread is a coherent sequence of finite-stage behaviors and carries no requirement of realization in $\Omega$; a coalgebra whose own tower empties in the limit may nonetheless sit inside a terminal sequence whose predicate tower retains threads. The coincidence criterion restates the definition of realization.
\end{proof}

Proposition~\ref{prop:refinement} fixes the division of labor between the two registers. The lattice trichotomy of Definition~\ref{def:status} is autonomous: it classifies the completed common trust \emph{of a given structure}. The functorial apparatus classifies the \emph{idealization}: the space of all trust-behaviors coherent to every finite depth, of which any given structure's tower is the pullback along its cone. Facts about the terminal sequence therefore constrain every structure at once, in the pullback sense---and it is in that sense, made exact here, that the convergence results of the next subsection bear on cautious common trust.

\subsection{The terminal-sequence picture and the caution obstruction}

By Proposition~\ref{prop:towercorr}, the subset tower $\{X_n\}$ is the pullback, along the canonical cone, of the trust-predicate tower on the terminal sequence of $F$, and liminality is a Mittag-Leffler phenomenon: $\Theta$ is liminal exactly when the tower has all stages nonempty but empty inverse limit, equivalently when its bonding maps fail eventual surjectivity---with Proposition~\ref{prop:refinement} fixing the precise sense in which the functorial classification bears on a given structure. The question is then when the functorial reading \emph{completes}---at what stage, if any, the terminal sequence stabilizes so that the idealized classification closes. Here the convergence theory for set functors supplies exact answers, and they divide sharply on caution. Two facts frame the result. First, by \citet{worrell2005}, the final sequence of a \emph{finitary} set functor converges in $\omega + \omega$ steps---not at $\omega$: the $\omega$-th object is a completion of the final coalgebra, and the subsequent $\omega$ steps prune it, one level at a time, to the final coalgebra itself; for the $\lambda$-bounded power-set functors the construction needs precisely $\lambda + \omega$ steps \citep{adamek2015}. Convergence \emph{at} $\omega$ would require preservation of $\omega^{\mathrm{op}}$-limits, a strictly stronger property than finitariness which the power-set-like functors lack. Second, a functor whose values outrun the cardinality of their arguments has no final coalgebra at all, by Lambek's fixed-point lemma \citep{lambek1968} together with Cantor's theorem.

\begin{theorem}[Caution obstructs functorial completion]\label{thm:caution}
The profligate nest functor $\Nest^{\mathrm{prof}}$, whose elements are nests with outermost sphere equal to the whole space, admits \emph{no final coalgebra}: its terminal sequence never converges, at any ordinal stage. Moreover profligacy is incompatible with $\kappa$-boundedness over spaces of size $\geq \kappa$, so no bounded (hence accessible) variant of the cautious functor exists over large state spaces. Consequently the functorial completion of cautious subjunctive common trust is not attained at any stage whatsoever; whereas dropping profligacy yields a finitary nest functor whose final sequence converges in $\omega+\omega$ steps, where the trichotomy of Definition~\ref{def:status} completes.
\end{theorem}

\begin{proof}
For any $X$ and any designated point $w$, every subset $S \ni w$ yields the profligate two-sphere nest $\{S, X\}$, so $|\Nest^{\mathrm{prof}}(X)| \geq 2^{|X|-1} > |X|$ for infinite $X$. If $Z$ were a final coalgebra of $F^{\mathrm{prof}}(X) = L \times \Nest^{\mathrm{prof}}(X)^2$, Lambek's lemma would give a bijection $Z \cong L \times \Nest^{\mathrm{prof}}(Z)^2$, contradicting $|\Nest^{\mathrm{prof}}(Z)| > |Z|$. Hence no final coalgebra exists and the terminal sequence never stabilizes. For the incompatibility claim: a strictly increasing nest of spheres each of size $<\kappa$ has length at most $\kappa$, hence union of size at most $\kappa$ (repetitions in a non-strict nest contribute nothing further to the union); profligacy demands union $= X$, impossible for $|X| > \kappa$, so bounding the spheres destroys profligacy over large spaces. (The earlier observation that $\Nest^{\mathrm{prof}}$ is not finitary---no profligate nest over infinite $X$ is supported on a finite subset, since the image of a nest over finite $X_0$ has union inside $X_0$---is subsumed: non-finitariness follows a fortiori from non-existence of the final coalgebra together with Worrell's convergence theorem for finitary functors.) Finally, the non-profligate nest functor with finite nests of finite spheres is finitary, being a subfunctor of a finite iterate of the finite power-set functor, and \citet{worrell2005} gives convergence of its final sequence in $\omega + \omega$ steps, at which stage the tower's classification---realizable, impossible, or liminal per the Mittag-Leffler criterion---is determined.
\end{proof}

The theorem is negative but sharply localizing, and sharper than a displacement result: the feature denying cautious trust a completed functorial status is not the nesting, nor the interactivity, but caution itself, and what caution denies is not merely completion at $\omega$ but completion at \emph{any} ordinal. There is no ``transfinite tail'' in which the classification of cautious completed common trust eventually closes; the sequence simply never stabilizes. For non-cautious trust the classification does close, in $\omega+\omega$ steps rather than $\omega$---a correction to the naive expectation that finitariness buys convergence at $\omega$, which conflates preservation of filtered colimits with preservation of $\omega^{\mathrm{op}}$-limits. The lattice-theoretic results of the preceding subsections (Proposition~\ref{prop:bivalence} and the trichotomy) remain autonomous facts about any given structure; by Propositions~\ref{prop:towercorr} and~\ref{prop:refinement}, what the functorial reading adds is a constraint on all structures at once, in the pullback sense: the idealized classification never closes in the cautious case, the coincidence criterion of Proposition~\ref{prop:refinement} is unavailable because there is no thread space against which realization could stabilize, and so a cautious truster's completed common trust is an object her representation can approximate at every finite depth but can never, at any stage of idealization, attain.

\subsection{The price and the demarcation}

Strategy~II keeps the counterfactual \emph{ontic}---the closeness ordering is a fact of the worldly structure, not of $A$'s belief---and it is precisely this refusal to doxasticize that prevents it from borrowing the BFK machinery and forces the construction of its own. What it gains is expressive range: by permitting the Limit Assumption to fail, it gives completed common trust a third place to live, unavailable to the finite-level lexicographic representation. What it loses is the quarantine of betrayal: a bare sphere is an unranked set, and the mere presence of a violating world in a sphere consulted by the counterfactual contaminates the evaluation, with no level or conditioning index to hold it at bay. This loss is exactly what Section~\ref{sec:betrayal} will show to be consequential.

\section{Strategy III: Subjunctive Trust as Strong Belief}\label{sec:cps}

\subsection{The representation}

A counterfactual is, before anything else, a claim of robustness under supposition; and exactly one attitude in epistemic game theory is built to express robustness under supposition: \emph{strong belief} \citep{battigalli2002}. Fix a space $\Omega$ and a collection $\mathcal{B}$ of nonempty conditioning events. A \emph{conditional probability system} (CPS) is a family $\mu(\cdot \mid B)$, $B \in \mathcal{B}$, each a probability measure with $\mu(B \mid B) = 1$, satisfying the chain rule
\begin{equation}\label{eq:chainrule}
\mu(A \mid C) = \mu(A \mid B)\, \mu(B \mid C) \qquad (A \subseteq B \subseteq C,\ B, C \in \mathcal{B}).
\end{equation}
An agent \emph{strongly believes} $E$ if
\begin{equation}\label{eq:strongbelief}
\mu(E \mid B) = 1 \quad \text{for every } B \in \mathcal{B} \text{ with } B \cap E \neq \varnothing,
\end{equation}
i.e.\ believes $E$ under every hypothesis consistent with it, relinquishing it only upon conditioning on something incompatible. We identify $A$'s subjunctive trust in competence with strong belief of $\Comp_P$, and in integrity with strong belief of $\Int_P$, the conditioning family containing the antecedents of the trust-counterfactuals and the contingencies under which trust is tested. This does not launder the counterfactual into a flat conditional: strong belief requires belief in competence precisely under the supposition that \emph{activates} it, together with persistence under every further consistent supposition.

\subsection{Non-inductiveness, brittleness, explicitation}

The three signature features are consequences rather than stipulations. \emph{Non-inductiveness}: strong belief is a property of the CPS, not a tally of frequencies. \emph{Brittleness}: if the actual world realizes the antecedent and falsifies the consequent, then conditioning on the actual hypothesis $B^\star$ consistent with the trust event gives $\mu(\Comp_P \mid B^\star) < 1$; and since strong belief requires probability one at \emph{every} consistent supposition, a single failure negates it outright---the all-or-nothing character that \citet{jones1996} captures in describing trust as an optimism that a single salient breach can extinguish where a mere expectation would only be revised. \emph{Explicitation}: strong belief is dispositional, specifying beliefs under suppositions not actually conditioned on; while no test occurs the off-supposition conditionals are latent and unexercised, which is the formal image of trust's implicitness. Here the strong-belief representation supplies exact content to \citet{holton1994}'s participant stance. Holton characterizes trust as a \emph{readiness} to feel betrayal should reliance be disappointed---not an occurrent belief but a standing disposition awaiting a triggering circumstance. The latent conditionals of a CPS \emph{are} that readiness: a family of commitments about what one would believe were the testing hypotheses to arise, held in advance of their arising. A disconfirming experience forces conditioning on the testing hypothesis $B^\star$, and that single conditioning operation is simultaneously the explicitation (the latent conditional---the readiness---is exercised) and the violation (it is found to fail). Where the lexicographic account needed two updates, one object-level and one introspective, the strong-belief account renders explicitation and violation as a \emph{single} conditioning operation---the sharpest of the three accounts, the closest formal realization of the participant stance, and the one that matters most in Section~\ref{sec:betrayal}.

\subsection{Inheritance: a forward-induction characterization}

\citet{battigalli2002} established that, in a complete type structure with CPSs, rationality and common strong belief of rationality characterizes extensive-form rationalizability---the forward-induction solution. Placing trustworthiness where rationality sits, define \emph{common strong belief of trustworthiness} as the hierarchy in which each agent strongly believes the other trustworthy, strongly believes the other strongly believes both trustworthy, and so on. The disconfirmation dynamics of Section~\ref{sec:cps} correspond to reaching an information set at which strong belief of trustworthiness can no longer be maintained.

Because trustworthiness is an event about the \emph{partner's} dispositions rather than the agent's own best response, the transposition requires proof rather than analogy. We supply it for the class of finite extensive-form games with observable disclosure moves, indicating at each step which element of the Battigalli--Siniscalchi machinery is invoked.

Fix such a game $\Gamma$ with players $A, B$, strategy sets $S_i$, and information sets $H_i$. For each player let the \emph{compliance set} $\hat S_i \subseteq S_i$ collect the strategies that disclose at every information set at which $i$ knows the relevant fact (and do not route); trustworthiness is the \emph{rectangular} event $\mathrm{TW}_i = \hat S_i \times T_i$. A conditional type structure $\mathcal{T} = (T_A, T_B, \beta_A, \beta_B)$ assigns to each type $t_i$ a conditional probability system $\beta_i(t_i)$ on $S_{-i} \times T_{-i}$ with the external conditioning family $\mathcal{B}_i = \{S_{-i}(h) \times T_{-i} : h \in H_i\}$; belief-completeness (surjectivity of $\beta_i$ onto CPSs) is available by the canonical construction of \citet{battigalli1999}. A pair $(s_i, t_i)$ is \emph{sequentially rational} if $s_i$ is optimal under $\beta_i(t_i)$ at every information set consistent with $s_i$; write $R_i$ for the set of such pairs. Define the epistemic hierarchy
\begin{equation}\label{eq:Qk}
Q_i^0 = R_i \cap \mathrm{TW}_i, \qquad Q_i^{k+1} = Q_i^k \cap \big\{(s_i, t_i) : \beta_i(t_i) \text{ strongly believes } Q_{-i}^k\big\},
\end{equation}
with $\mathrm{RCSBT} = \bigcap_k (Q_A^k \times Q_B^k)$, and the strategic elimination procedure---\emph{iterated trustworthiness}---by
\begin{equation}\label{eq:Sigmak}
\Sigma_i^0 = \big\{ s_i \in \hat S_i : s_i \text{ sequentially rational for some CPS on } S_{-i} \big\}, \quad
\Sigma_i^{k+1} = \big\{ s_i \in \Sigma_i^k : s_i \text{ seq.\ rational for some CPS strongly believing } \Sigma_{-i}^k \big\}.
\end{equation}

The intersection with the previous round in \eqref{eq:Qk} is not stylistic: strong belief is not monotone---$E \subseteq F$ does not yield that strong belief of $E$ implies strong belief of $F$---so the hierarchy must conjoin each level with strong belief of the previous full conjunction, exactly as in \citet{battigalli2002}; the alternating recursion \eqref{eq:chain} has this shape by construction.

\begin{theorem}[Iterated trustworthiness characterization]\label{prop:closure}
Let $\Gamma$ be a finite extensive-form game with observable disclosure moves and $\mathcal{T}$ a belief-complete conditional type structure for $\Gamma$ with external conditioning family. Then for every $k$, $\mathrm{proj}_{S_i}\, Q_i^k = \Sigma_i^{k}$, and hence $\mathrm{proj}_S\, \mathrm{RCSBT} = \Sigma^\infty = \bigcap_k \Sigma^k$. Moreover $\mathrm{RCSBT} \neq \varnothing$ if and only if $\Sigma^0 \neq \varnothing$, i.e.\ if and only if disclosure-compliance is sequentially rational, at every information set at which the discloser knows, for some conditional probability system on the opponent's strategies.
\end{theorem}

\begin{proof}
\emph{Forward inclusion} ($\mathrm{proj}\, Q^k \subseteq \Sigma^k$), by induction. At $k = 0$: if $(s_i, t_i) \in R_i \cap \mathrm{TW}_i$ then $s_i \in \hat S_i$ by rectangularity of $\mathrm{TW}_i$, and $s_i$ is sequentially rational for the marginal of $\beta_i(t_i)$ on $S_{-i}$, which is a CPS on $S_{-i}$ because the conditioning events are external. For the step, it suffices that marginalization preserves strong belief: if $\beta_i(t_i)(Q_{-i}^k \mid B) = 1$ for every external $B$ meeting $Q_{-i}^k$, then the marginal assigns probability one to $\mathrm{proj}_{S_{-i}} Q_{-i}^k$ conditional on every strategy-event meeting it, since projections only enlarge events and the external events are cylinders over strategy sets; by the inductive hypothesis $\mathrm{proj}\, Q_{-i}^k \subseteq \Sigma_{-i}^k$, and monotonicity of probability-one under enlargement gives strong belief of $\Sigma_{-i}^k$.

\emph{Converse inclusion} ($\Sigma^k \subseteq \mathrm{proj}\, Q^k$), by the witness construction of \citet{battigalli2002}, adapted. By induction construct, for each $s_i \in \Sigma_i^k$, a type $t^k(s_i)$ with $(s_i, t^k(s_i)) \in Q_i^k$: the definition of $\Sigma_i^k$ supplies a CPS $\mu$ on $S_{-i}$ strongly believing $\Sigma_{-i}^{k-1}$ for which $s_i$ is sequentially rational; lift $\mu$ to a CPS on $S_{-i} \times T_{-i}$ by composing with the inductively constructed assignment $s_{-i} \mapsto t^{k-1}(s_{-i})$ on the survivors (and arbitrarily elsewhere), so that the lift strongly believes $Q_{-i}^{k-1}$; belief-completeness supplies a type inducing the lift. The single obligation beyond the rationality case is membership in $\mathrm{TW}_i$ at round $0$---the constructed type's sequentially rational strategy must be \emph{compliant}---and this is guaranteed by, and only by, $s_i \in \Sigma_i^0 \subseteq \hat S_i$: the construction places no compliant strategy into $Q_i^0$ unless compliance is sequentially rational for its CPS, which is the content of $\Sigma_i^0 \neq \varnothing$. The non-vacuity claim follows: if $\Sigma^0 = \varnothing$ then $Q^0 = \varnothing$ and $\mathrm{RCSBT} = \varnothing$; conversely if $\Sigma^0 \neq \varnothing$ the construction populates every $Q^k$, and finiteness of $\Gamma$ stabilizes the sequences.
\end{proof}

\begin{corollary}[The payoff residue, parametrically]\label{cor:residue}
In the two-stage disclosure game $\Gamma(\pi, \tau, q, v)$---Nature draws the fact with probability $\pi$; $B$ observes it; $B$ discloses to $A$, routes to $C$ for a temptation payoff $\tau > 0$, or stays silent; routing forfeits the continuation value $v$ with detection probability $q$---compliance is sequentially rational at the knowing information set if and only if
$\tau \leq q\,v$,
and hence $\mathrm{RCSBT} \neq \varnothing$ if and only if $\tau \leq q v$. When $\tau > q v$, $\mathrm{RCSBT} = \varnothing$ while the purely epistemic hierarchy---common strong belief of compliance, without the rationality conjunct---remains satisfiable: a type may strongly believe the partner's compliance whether or not compliance is optimal for anyone. Emptiness of the joint event is thus a payoff phenomenon, not an epistemic one.
\end{corollary}

\begin{proof}
At the information set where $B$ knows, routing yields $\tau + (1-q)v$ and disclosing yields $v$; disclosure is a best response iff $v \geq \tau + (1-q)v$ iff $\tau \leq qv$, and no belief about $A$ affects the comparison, so the condition is necessary and sufficient for $\Sigma_B^0 \neq \varnothing$; Theorem~\ref{prop:closure} converts this into the non-vacuity of $\mathrm{RCSBT}$. For the final claim, the hierarchy \eqref{eq:Qk} with $R_i$ deleted from round $0$ is populated by types constructed exactly as in the converse direction of the theorem, with no optimality obligation to discharge.
\end{proof}

The residue is worth dwelling on, because it is where ``trustworthiness is not rationality'' does real work. The closure check does not fail; it \emph{relocates} the question from logic to economics. Whether rational trustworthy types exist is a condition on the game's payoffs---on whether the strategic environment makes trustworthiness affordable---and not on the epistemic apparatus. This is, in formal dress, exactly the thesis of \citet{hardin2002}'s encapsulated-interest account: trust is warranted when the trusted party's own interests encapsulate the truster's, so that trustworthy behavior is what the trustee's incentives themselves recommend. Theorem~\ref{prop:closure} says where that thesis lives in the epistemic architecture: not in the definition of trust, nor in the hierarchy of attitudes, but precisely and only in the non-vacuity condition of the joint event---RCSBT is populated exactly when interests encapsulate, and empty otherwise, with Corollary~\ref{cor:residue} computing the encapsulation condition, $\tau \leq qv$, in closed form for the parametric game. It also marks the thesis's limit: encapsulation governs whether common trust is \emph{inhabited}, not whether it is \emph{coherent}, and a theory of trust that could not register this distinction, between an incoherent common attitude and a coherent one that the incentives happen to leave unpopulated, would be missing precisely what makes institutional design for trust a substantive problem.

\subsection{Which diagonals close}

The characterization presupposes a complete type structure, and completeness is what \citet{brandenburger2006} attacks. Whether the impossibility refutes completed common strong-belief trust is more delicate than in the bare-assumption case, because the measure-coherence of conditioning changes which diagonals can be built. Let $\SB_a(t)$ be the set of Bob-types that Ann-type $t$ strongly believes, and symmetrically $\SB_b$. The first diagonal mirrors the assumption case, $D = \{t : t \notin (\SB_b \circ \SB_a)(t)\}$.

\begin{proposition}[Strong-belief-set diagonal fails by vacuity]\label{prop:vacuity}
The strong-belief-set diagonal does not close. For the contradiction to close, a Bob-type $s^\star$ must strongly believe $D$, and the closing Ann-type realizes a conditioning event $B_0 = \{t : \SB_a(t) = \{s^\star\}\}$ at which the diagonal condition reduces to $t \in D \iff t \notin D$, so $B_0 \cap D = \varnothing$. Strong belief of $D$ is then vacuously waived at $B_0$---the all-suppositions quantifier requires probability one only at conditioning events consistent with $D$---so $s^\star$ is not forced to strongly believe $D$ where the contradiction needs it.
\end{proposition}

This vacuity has no analogue in the bare-assumption setting, where assumption of a single set carries no quantifier to waive; it is a dividend of the very robustness-across-suppositions that recommended strong belief for individual trust.

\subsection{The central result: the conditioning family}

A second diagonal evades the vacuity by working in a single conditional slice. For a fixed $B$, let $c^B_b(s)$ be the set of Ann-types Bob-type $s$ deems possible conditional on $B$, and form $D_B = \{t : t \notin (c^B_b \circ c^{B'}_a)(t)\}$. A single slice has no quantifier to waive, so the construction threatens to reproduce the bare-assumption contradiction inside one conditional. Because the impossibility theorem of \citet{brandenburger2006} is itself stated relative to a language---their diagonal set is definable in the first-order language of the belief structure, and completeness is completeness \emph{for} that language---the analysis of this diagonal must likewise be language-relative, and we fix the relevant collections before stating the result.

Let $\mathcal{L}^c$ be the first-order language of the two-sorted structure $(T_A, T_B)$ equipped, for each conditioning event $B \in \mathcal{B}$, with a relation $\mathsf{P}^B_i(t, u)$ read ``conditional on $B$, type $t$ deems partner-type $u$ possible''---the conditional-map vocabulary; write $\mathrm{Def}(\mathcal{L}^c)$ for the collection of subsets definable by $\mathcal{L}^c$-formulas with parameters. Let $\mathcal{B}_{\mathrm{ext}}$ be the \emph{external algebra}: events of the form $\hat{E} \times T_{-i}$ with $\hat{E}$ an observable strategy or history event---cylinders over the external coordinate, carrying no type information. Two facts about the candidate conditioning event $B_0 = \{t : c^{B'}_a(t) = \{s^\star\}\}$ should be separated at the outset. First, $B_0$ \emph{is} $\mathcal{L}^c$-definable with parameters: $B_0 = \{t : \forall u\, (\mathsf{P}^{B'}_A(t,u) \leftrightarrow u = s^\star)\}$. Expressibility is therefore not where the analysis turns. Second, the parameter chain of that definition is cyclic---$s^\star$ is specified through $D_{B_0}$, which is specified through $B_0$---so the definition determines an object only where the chain closes, namely at a fixed point of the operator below; and for the diagonal to be \emph{conditioned on}, that fixed point must lie in the conditioning family. The question is thus not definability but admissibility as a conditioning event.

\begin{proposition}[Diagonal-fate equals fixed-point existence]\label{prop:fixedpoint}
The conditional-map diagonal generates the Brandenburger--Keisler contradiction if and only if the operator
\begin{equation}
\Phi(B) = \{t : c^{B'}_a(t) = \{s^\star(B)\}\}, \qquad c^B_b(s^\star(B)) = D_B,
\end{equation}
has a fixed point $B_0$ lying in the conditioning family $\mathcal{B}_b$.
\end{proposition}

\begin{proof}
The witness $s^\star$ is the Bob-type whose $B_0$-slice supports exactly $D_{B_0}$; the closing type $t^\star$ must lie in $B_0$, whereupon $t^\star \in D_{B_0} \iff t^\star \notin c^{B_0}_b(s^\star) = D_{B_0}$. The definition of $D_{B_0}$ is parametric in $B_0$, which is parametric in $s^\star$, which is parametric in $D_{B_0}$; the chain closes, and the three objects are simultaneously determined, exactly at a fixed point of $\Phi$. If such a fixed point lies in $\mathcal{B}_b$, then $D_{B_0}$ is a determinate $\mathcal{L}^c$-definable set, $B_0$ is an admissible conditioning event, $t^\star$ exists by belief-completeness, and the displayed equivalence is a genuine contradiction: no such structure exists. If no fixed point lies in $\mathcal{B}_b$, then no admissible conditioning event fills the diagonal's conditioning slot, and no instance of the contradiction is derivable.
\end{proof}

Whether a fixed point of $\Phi$ belongs to $\mathcal{B}_b$ depends on a closure condition that must be distinguished from the one the impossibility theorem attacks:
\begin{itemize}[nosep]
\item \emph{Belief-completeness}: every CPS on the partner's strategy-type space, with conditioning family $\mathcal{B}$, is realized by some type. This is the notion the construction of \citet{battigalli1999} secures and the analogue of the completeness \citet{brandenburger2006} attack.
\item \emph{Conditioning-completeness for $\mathcal{L}^c$}: $\mathcal{B}_b \supseteq \mathrm{Def}(\mathcal{L}^c)$---every definable set of partner-types, the impredicatively specified ones included, is an admissible conditioning event.
\end{itemize}
These quantify over different collections---the first over belief systems relative to a fixed $\mathcal{B}$, the second over the contents of $\mathcal{B}$ itself---and they interact in opposite directions, as the following lemma makes exact.

\begin{proposition}[Cylinder lemma]\label{lem:cylinder}
Let $\mathcal{T}$ be belief-complete with $\mathcal{B}_b \subseteq \mathcal{B}_{\mathrm{ext}}$. Then no fixed point of $\Phi$ lies in $\mathcal{B}_b$. Indeed, belief-completeness itself supplies the exclusion: for any candidate $B_0$, completeness realizes two Ann-types with identical external coordinates whose conditional maps at $B'$ differ---one equal to $\{s^\star\}$, one not---so $B_0$ separates two points of a single external fiber, is not a cylinder over the external coordinate, and hence lies outside $\mathcal{B}_{\mathrm{ext}}$.
\end{proposition}

\begin{proof}
An element of $\mathcal{B}_{\mathrm{ext}}$ contains, for each external coordinate value, either every type with that coordinate or none. Fix any candidate $B_0 = \{t : c^{B'}_a(t) = \{s^\star\}\}$. Belief-completeness realizes, over any single external coordinate, a type whose CPS conditional on $B'$ is the point mass at $s^\star$ and a type whose CPS conditional on $B'$ charges some other partner-type; both CPSs exist as coherent systems on the fixed family $\mathcal{B}$, so both types exist. $B_0$ contains the first and omits the second, separating a fiber; so $B_0 \notin \mathcal{B}_{\mathrm{ext}} \supseteq \mathcal{B}_b$. The argument is uniform in the parameter $s^\star$ and hence applies to every candidate, fixed point or otherwise.
\end{proof}

\begin{theorem}[Conditional escape]\label{thm:escape}
Completed common strong-belief trust escapes the Brandenburger--Keisler diagonals if and only if the conditioning family is not conditioning-complete for $\mathcal{L}^c$---specifically, if and only if no fixed point of $\Phi$ lies in $\mathcal{B}_b$. The strong-belief-set diagonal having failed by vacuity (Proposition~\ref{prop:vacuity}), the impossibility is generated, if at all, only through the conditional-map diagonal, and only when conditioning-completeness supplies its witness.
\end{theorem}

The antecedent of Theorem~\ref{thm:escape} is now discharged for the canonical framework by proof rather than by appeal to existence. In that framework, hierarchies of conditional probability systems are constructed with respect to a fixed collection of relevant hypotheses concerning an external state \citep{battigalli1999}---in extensive-form applications, the events generated by the histories of the game---so that $\mathcal{B} \subseteq \mathcal{B}_{\mathrm{ext}}$ by definition. Proposition~\ref{lem:cylinder} then yields the exclusion of every candidate $B_0$ outright, with belief-completeness serving as premise rather than as conclusion: the richer the structure's beliefs, the more decisively the type-based event is expelled from the external algebra. The constructive existence of the universal belief-complete structure \citep{battigalli1999} now enters with its correct logical role---as the instantiation showing that the consistent case is inhabited---rather than as the ground of the escape. The two completeness notions are thereby kept apart as they must be: the terminality of the universal structure is terminality in the category of structures conditioned on external hypotheses, while the completeness the impossibility theorem attacks quantifies over a definability collection that, by the cylinder lemma, the external algebra cannot contain.

\begin{corollary}[Unconditional escape in the canonical framework]\label{cor:escape}
In conditional-probability-system type structures with conditioning events restricted to external hypotheses \citep{battigalli1999, battigalli2002}, completed common strong-belief trust is not refuted by the Brandenburger--Keisler construction: by Proposition~\ref{lem:cylinder}, no candidate witness event is admissible, so the escape of Theorem~\ref{thm:escape} holds unconditionally. The boundary of the result is exact: the question reopens only for enriched frameworks in which epistemic (type-based) events are admitted as conditioning hypotheses.
\end{corollary}

This is a genuine asymmetry with assumption-trust, for which the diagonal closes on belief-completeness alone. The feature responsible is precisely the one that made strong belief the faithful representation of individual trust: its anchoring in a \emph{given} family of suppositions about how the world might be, rather than an unconditional or fully-completed family that would include suppositions about the believer's own kind.

\section{How Each Representation Survives Betrayal}\label{sec:betrayal}

We now return to the three episodes of Section~1 and show that the representations survive them differently. Recall \citet{baier1986}'s founding observation: what separates trust from mere reliance is that trust can be \emph{betrayed} and not merely disappointed. The formal work of this section can be read as an anatomy of that distinction. Disappointment, in each of our representations, is the failure of an expectation---an event assigned low or moderate probability occurring; it revises. Betrayal is something structurally different in each: the refutation of an \emph{ordering}, the contamination of a \emph{closeness structure}, the exercise-and-failure of a \emph{standing conditional}. Baier's distinction, which she drew with the contrast between a shelf that gives way and a person who lets one down, thus turns out not to be one distinction but three, depending on the form in which the trust was held---and the differences are not incidental; they trace directly to the structural features established above---the level--depth budget, the failure of quarantine, and the single-conditioning-operation character of strong belief.

\subsection{The epistemic location of a violator}

The organizing observation is that a betrayal's severity, under each representation, depends on \emph{where} the violating world sits in the agent's internal structure, and that deep (knowing) and shallow (ignorant) violators sit differently. In Example~\ref{ex:incompetent} the violator is shallow: $B$ simply fails to know, and a world in which $B$ is incompetent is intuitively \emph{distant} for a truster who expected competence. In Examples~\ref{ex:routed} and~\ref{ex:silence} the violator is deep: $B$ is competent and correctly models $A$, so the violating world is one in which almost all trust conditions hold and only the final disclosure step fails---an intuitively \emph{near} world.

\subsection{A worked model}\label{sec:model}

To make the three mechanisms exact rather than analogical, we fix a single six-state structure realizing the routed-confidence episode and equip it, in turn, with each of the three representations; the betrayal update is then computed, not described. States specify whether $B$ knows $P$ (K), discloses to $A$ (T), routes to $C$ (R), correctly models $A$'s valuation of $C$ (M), and $B$'s reciprocal trust depth $r$:
\begin{center}
\begin{tabular}{lccccl}
 & K & T & R & M & \\[1pt]
$\omega_1$ & \checkmark & \checkmark & -- & \checkmark & $r=2$ \\
$\omega_2$ & \checkmark & \checkmark & -- & \checkmark & $r=1$ \\
$\omega_3$ & \checkmark & \checkmark & -- & \checkmark & $r=0$ \\
$\omega_4$ & -- & -- & -- & \checkmark & honest incompetent \\
$\omega_5$ & -- & -- & -- & -- & honest incompetent \\
$b$ & \checkmark & -- & \checkmark & \checkmark & knowing betrayal \\
\end{tabular}
\end{center}
Thus $\Comp = \{\omega_1, \omega_2, \omega_3, b\}$, $\Int = \{\omega_1, \dots, \omega_5\}$ ($\omega_4, \omega_5$ vacuously), $\Tw^{(1)} = \{\omega_1,\omega_2,\omega_3\}$, $\Tw^{(2)} = \{\omega_1,\omega_2\}$, $\Tw^{(3)} = \{\omega_1\}$; shells $W_3 = \{\omega_1\}$, $W_2 = \{\omega_2\}$, $W_1 = \{\omega_3\}$, $W_\bot = \{\omega_4, \omega_5, b\}$. One feature of the construction is forced by the episode itself and drives everything that follows: the betrayal state $b$ shares both epistemic attributes, K and M, with the deeply trustworthy states. Deliberateness consists in possessing the epistemic marks of trustworthiness.

\emph{Lexicographic reading.} Theorem~\ref{thm:leveldepth} with $d = 3$, cautious, $W_\bot \neq \varnothing$ prescribes $n_d + 1 = 4$ levels; take the staircase $\mu_0 = \delta_{\omega_1}$, $\mu_1 = \delta_{\omega_2}$, $\mu_2 = \delta_{\omega_3}$, $\mu_3 = (0.45, 0.45, 0.10)$ on $(\omega_4, \omega_5, b)$, so that within the residue incompetence is deemed likelier than betrayal. Direct inspection confirms that each $\Tw^{(k)}$ is assumed. Upon learning $b$, lexicographic conditioning selects the first level charging $\{b\}$, namely $\mu_3$, with posterior $\delta_b$ \citep{blume1991a}. The structural consequence is the re-shelving of Section~\ref{sec:betrayal}: the shell partition was generated by the sorting rule ``K and M and reciprocal-trust attributes place a state high,'' and $b$, which satisfies K and M yet was filed in $W_\bot$, refutes the rule rather than a single filing. Re-articulating trust requires a refining attribute (betrayal-despite-M) that splits each shell in two, and over any enrichment of the state space in which the refined shells are nonempty, the level count demanded by Theorem~\ref{thm:leveldepth} rises to as much as $2n_d + 1$. The four-level structure is not updated but outgrown; the six-state model exhibits the mechanism, the enriched space its full arithmetic.

\emph{Ordinal reading.} Let closeness be governed by the similarity order $\delta(\omega) = |\{\mathrm{K}, \mathrm{M}\} \cap \omega|$, the count of epistemic attributes shared with the trusted profile; we record this as an explicit assumption, and note that any similarity order that respects epistemic attributes yields the same qualitative placement. Then $\delta(\omega_1) = \delta(\omega_2) = \delta(\omega_3) = \delta(b) = 2 > \delta(\omega_4) = 1 > \delta(\omega_5) = 0$, and $A$'s trust consists in the sphere assignment $S^0 = \{\omega_1, \omega_2, \omega_3\}$, $S^1 = S^0 \cup \{b\}$, $S^2 = S^1 \cup \{\omega_4\}$, $S^3 = \Omega$: trust is exactly the exclusion of $b$ from the innermost sphere, and the similarity order forces $b$ into the very next one. The claim that deliberateness is closeness is thus not a postulate of the analysis but a consequence of a stated similarity assumption: the deliberate betrayer sits at the first sphere outside the trusted core \emph{because} deliberateness is epistemic-attribute sharing. Before the betrayal, the innermost $\llbracket K \rrbracket$-meeting sphere is $S^0$ and $S^0 \cap \llbracket K \rrbracket \subseteq \Int$: the integrity counterfactual holds, by one sphere's margin. Upon the revelation of $b$, strong centering \citep{lewis1973} places the actual state innermost; every trustworthy state's innermost $K$-meeting sphere now contains $b \notin \Int$, the operator $\mathsf{Tr}$ fails at $\omega_1, \omega_2, \omega_3$ simultaneously, and the tower collapses at the first stage---total contamination in one step, because $b$ was near. The contrast with the shallow violator is equally computable: learning $\omega_4$ instead, one finds $\omega_4 \notin \llbracket K \rrbracket$, the antecedent-meeting spheres of the integrity counterfactual are untouched, and integrity trust survives intact while competence alone is revised. The extent of contamination equals the attribute overlap: total at $\delta = 2$, partial at $\delta = 1$.

\emph{Strong-belief reading.} Take the conditioning family of external, observable hypotheses $\mathcal{B} = \{\Omega, H_1, H_2, H_3, H_2 \cup H_3\}$, where $H_1$ (``disclosed to $A$'') $= \{\omega_1,\omega_2,\omega_3\}$, $H_2$ (``routed to $C$'') $= \{b\}$, $H_3$ (``no disclosure'') $= \{\omega_4, \omega_5\}$, in accordance with Corollary~\ref{cor:escape}. Let $A$'s CPS be $\mu(\cdot \mid \Omega) = (0.6, 0.3, 0.1)$ on $(\omega_1, \omega_2, \omega_3)$; $\mu(\cdot \mid H_3) = (0.5, 0.5)$ on $(\omega_4, \omega_5)$; $\mu(\cdot \mid H_2 \cup H_3) = (0.5, 0.5, 0)$ on $(\omega_4, \omega_5, b)$; $\mu(\cdot \mid H_2) = \delta_b$. The chain rule holds, the only nontrivial instance being $H_3 \subseteq H_2 \cup H_3$, where $\mu(H_3 \mid H_2 \cup H_3) = 1$. Strong belief of $\Int$ is verified clause by clause: probability one given $\Omega$, $H_1$, $H_3$, and---the substantive clause---given $H_2 \cup H_3$, where conditional on \emph{something having gone wrong} $A$ maintains that it was incompetence rather than betrayal; that maintenance is what strong belief of integrity consists in. The clause at $H_2$ is waived, $H_2 \cap \Int = \varnothing$. Upon observing $H_2$, since $H_2$ is inconsistent with integrity, strong belief was never asserted there; $A$ moves to the slice $\mu(\cdot \mid H_2) = \delta_b$---a conditional her system carried coherently all along---and every other conditional is untouched. Her posterior stance, reasoning under the hypothesis of knowing betrayal, is well-defined because $H_2$ together with $b$'s M-attribute constitutes that higher-order hypothesis. One conditioning operation; no structural refutation; a standing readiness exercised.

The model thus yields, alongside the three mechanisms, a quantitative contrast unavailable to informal comparison: re-shelving raises the level bound of Theorem~\ref{thm:leveldepth} by up to a factor of two; contamination extent equals epistemic-attribute overlap; and the conditional-probability system is invariant except at the single slice the betrayal activates. The subsections that follow interpret these computations.

\subsection{Strategy I: betrayal as re-shelving}

Under the lexicographic representation, trust placed trustworthy worlds at the top levels and quarantined violators at lower levels, the levels ordered by the shells $W_k$ of Theorem~\ref{thm:leveldepth}. Consider Example~\ref{ex:routed}. The betrayal world $\omega^\star$---$B$ competent, $B$ knowing $A$'s view of $C$, $B$ routing to $C$---is a \emph{deep-shell} violator: it satisfies competence and the second-order modelling condition, failing only the final disclosure. $A$ had therefore filed $\omega^\star$ under a high shell, close to the trustworthy region, treating betrayal-given-competence-and-correct-modelling as a remote lower level.

Conditioning on $\omega^\star$ promotes its level to primary. But because the LPS ordered worlds by depth-of-violation, and $\omega^\star$ violates only at a deep level while sitting epistemically near the top, the update reveals that \emph{the shell structure itself was mis-specified}: a world $A$ had treated as deep-trustworthy was in fact an integrity violator. The consequence is distinctive: by Theorem~\ref{thm:leveldepth}, $A$'s articulation of common trust to depth $d$ rested on $d$ levels, one per shell; the knowing betrayal shows the partition into shells was wrong, so $A$ does not merely demote one world but must \emph{re-sort} the space and rebuild the level ordering. In the lexicographic picture, knowing betrayal is a \emph{re-shelving catastrophe}: not a low-probability event that occurred, but a refutation of $A$'s qualitative ordering of possibilities, and with it the level budget that supported iterated trust.

Example~\ref{ex:incompetent}, by contrast, is mild under Strategy~I. The honest incompetent violates at the \emph{shallow} shell $W_1$ ($B$ not even depth-one competent), which the LPS had already placed at a low level. Conditioning on it promotes a level that was already near the bottom; the shell structure is undisturbed, and $A$ revises competence without re-shelving. The lexicographic representation thus predicts that \emph{ignorant} betrayal is absorbed while \emph{knowing} betrayal is structurally disruptive---a genuine and non-obvious asymmetry.

\subsection{Strategy II: betrayal as contamination, scaled by deliberateness}

Under the ordinal representation there is no quarantine: a violating world in a sphere consulted by the counterfactual contaminates the evaluation. Consider Example~\ref{ex:routed} again. The integrity counterfactual has antecedent $K_B P$; the betrayal world $\omega^\star$ satisfies it ($B$ knows $P$) and is \emph{near}, because $B$ is competent and deliberate---and competence and deliberateness are exactly what make a world close in the ordering. So $\omega^\star$ lies in the innermost $K_B P$-sphere, and $A$'s trust-counterfactual was \emph{already false} in the sphere structure; the discovery does not change the spheres but reveals that $A$ misjudged closeness.

The distinctive consequence follows from the contamination analysis behind Theorem~\ref{thm:caution}. Because $A$ knew $B$ knew $A$'s view of $C$---a depth-two condition---the betrayal violates the tower at a \emph{near} world, so the contamination front begins near the top and propagates inward through every depth of the common-trust tower at once. The ordinal representation therefore predicts that \emph{knowing} betrayal is maximally destructive precisely because deliberateness is closeness: the more knowing the betrayal, the nearer the violating world, the more central its position, the more of the trust tower it contaminates. A \emph{bumbling} betrayal (Example~\ref{ex:incompetent}, an honest failure to know) is a \emph{far} world and contaminates little---$A$ can still trust $B$'s integrity, and much of the tower survives---whereas the knowing betrayal of Example~\ref{ex:routed} is near and contaminates the whole. Where Strategy~I predicts an asymmetry between ignorant and knowing betrayal in the \emph{kind} of damage (absorbed versus re-shelving), Strategy~II predicts an asymmetry in the \emph{extent} of damage, scaled continuously by the betrayer's deliberateness. Example~\ref{ex:silence}, the competent paternalistic silence, is intermediate: $B$ is competent (near) but the withholding is benevolent, so whether $\omega^\star$ sits innermost depends on whether $A$'s ordering treats benevolent withholding as close---a modelling choice the episode underdetermines, and precisely the sort of choice the ordinal representation makes salient.

If the betrayal pushes the contamination through the whole tower, the re-established trust structure may fail well-foundedness, and by Definition~\ref{def:status} completed re-trust may be \emph{liminal}: approached at every finite depth but never attained. This is a formal rendering of the familiar report that one can never \emph{fully} trust again after a knowing betrayal---not that re-trust is impossible, but that it is a limit one approaches without closing.

\subsection{Strategy III: betrayal as a shift of the operative conditional}

Under the strong-belief representation the higher-order structure of Example~\ref{ex:routed} is handled \emph{natively}. The disclosure-to-$C$ is the surface event, but the episode specifies that $A$ knows $B$ knew $A$'s valuation of $C$; in the CPS framework this is a single, well-defined conditioning hypothesis $B^\star = $ ``$B$ discloses to $C$ while knowing $A$ regards $C$ as malicious.'' Such higher-order conditioning events are exactly what CPS structures carry---this is the forward-induction machinery for conditioning on deliberate ``surprising'' moves.

$A$ strongly believed $B$'s integrity: $\mu(\Int_P \mid B) = 1$ for every $B$ consistent with integrity. Is $B^\star$ consistent with integrity? No---it directly contradicts it. By \eqref{eq:strongbelief}, strong belief was therefore \emph{never required to hold conditional on $B^\star$}; the betrayal is the realization of a hypothesis incompatible with integrity, the one circumstance under which strong belief is relinquished without incoherence. The consequence is that strong belief \emph{degrades gracefully}. Learning that the incompatible $B^\star$ obtained does not refute $A$'s conditional structure; it moves $A$ to conditioning on $B^\star$, where strong belief of integrity correctly was never asserted. $A$'s CPS survives the betrayal as a coherent object: $A$ updates to the $B^\star$-slice and continues, now reasoning under ``$B$ is the type who, knowing the stakes, betrays''---a well-defined posterior strategic stance, and, by the forward-induction reading, one that tells $A$ how to treat all subsequent behaviour of $B$.

The contrast is sharp. Strategy~I suffers a re-shelving catastrophe (the level ordering is refuted); Strategy~II suffers contamination through the tower (the closeness ordering is refuted, and re-trust may be liminal); Strategy~III suffers only a \emph{shift of the operative conditioning hypothesis}, the belief object remaining coherent. Moreover the higher-order layer that makes the betrayal \emph{worse} under Strategy~II makes it \emph{more tractable} under Strategy~III: the ``knowing'' quality sharpens the conditioning event into a well-defined higher-order hypothesis the CPS already indexed, so $A$ knows exactly which slice to move to. The same feature---$B$'s knowing the stakes---is structurally disruptive under I, maximally destructive under II, and informative under III. For Examples~\ref{ex:silence} and~\ref{ex:incompetent} the pattern repeats with the appropriate conditioning events (``$B$ knew and chose silence,'' ``$B$ did not know''), each a distinct hypothesis on which $A$'s CPS carries a latent conditional that the episode exercises; in each case $A$'s belief structure survives and the operative conditional shifts.

\subsection{Summary of the differences}

Table-free, the three differences are these. \emph{On what betrayal refutes}: Strategy~I refutes the agent's qualitative level ordering; Strategy~II refutes her closeness ordering, with severity scaling with the betrayer's deliberateness; Strategy~III refutes nothing structural, shifting only the operative conditioning hypothesis. \emph{On how deliberateness matters}: knowing betrayal is structurally disruptive (I), maximally destructive because deliberateness is nearness (II), and informative because it sharpens a higher-order conditioning event (III). \emph{On recoverability}: Strategy~III leaves a coherent posterior stance on which trust can be reconstructed; Strategy~I requires rebuilding the level ordering; Strategy~II may leave only a liminal, never-fully-attained re-trust. If one wants a representation in which betrayal is survivable and informative rather than structurally catastrophic, the strong-belief representation is the one to reach for, and the reason is precisely its native handling of higher-order conditioning events---the same feature that, in Section~\ref{sec:cps}, conditioned its escape from the impossibility.

\section{Discussion}

\subsection{One axis organizes the three strategies}

The three representations divide on a single axis: how each treats the conditioning or closeness structure of the counterfactual. The lexicographic strategy reads it as a finite level order and inherits the impossibility by identity, its finiteness foreclosing the third status but affording the level--depth budget. The ordinal strategy keeps it as a primitive closeness ordering and can, by failing well-foundedness, give completed trust a liminal third status, at the cost of the quarantine of betrayal. The strong-belief strategy anchors it in a \emph{given} family of suppositions, and it is exactly the givenness---the restriction of conditioning events to external hypotheses---that secures the escape (unconditionally, in the canonical framework, by Corollary~\ref{cor:escape}) and that makes betrayal a graceful shift rather than a structural refutation. In each case the structural feature that makes the individual attitude a faithful model of trust is the same feature that governs the fate of \emph{common} trust and the survivability of betrayal: the level order's finiteness, the ordinal's non-well-foundedness, the conditioning family's givenness. The axis also sorts the informal literature. The truster of the rational-choice tradition \citep{gambetta1988, hardin2002}, who holds trust as a graded expectation disciplined by incentives, is at home in the lexicographic representation, where trust is a ranked hypothesis and Hardin's encapsulation condition reappears as the payoff residue of Theorem~\ref{prop:closure}. The truster of the counterfactual reading, for whom trust is a fact about which worlds are near, is the ordinal truster of \citet{lewis1973} and \citet{stalnaker1968}. And the truster of \citet{holton1994}'s participant stance, whose trust is a standing readiness rather than a belief, is the strong-belief truster, her readiness the latent conditionals of a CPS. That these three trusters coexist in any population---and within one person across domains, as Section~1 urged---is why the formal question ``what is betrayal?'' has three answers rather than one.

\subsection{Consequences for the epinet program}

For the epinet program \citep{moldoveanu2011, moldoveanu2014} the lesson is that trust's two faces reside in different parts of the formalism and must not be conflated. The \emph{reflective} face---an agent's own trust dispositions and her access to them, the implicit/explicit distinction, the phenomenology of betrayal---is local, resident at the nodes. The \emph{structural} face---whether completed common trust is realizable, impossible, liminal, or conditionally escapes---is global, resident in the composition of the edges' assumption, closeness, or conditioning maps. Betrayal acts on both: it explicitates the node-local disposition and, depending on the representation, refutes, contaminates, or shifts the edge-global structure. An epinet is precisely the object that holds both faces, and the analyses above are, in effect, statements about how the two layers interact under each representation of the trust relation.

\subsection{Limitations and open points}

Several points flagged as open in earlier drafts of this work have now been discharged, and the accounting should be exact about which. The closure check for trustworthiness in the forward-induction characterization is discharged by Theorem~\ref{prop:closure}: measurability, iteration shape, and witness construction all transfer, and the residue is a payoff condition---whether the game admits sequentially rational trustworthy behavior---which is a substantive economic question rather than a logical gap. The conditional escape is discharged for the canonical framework by the cylinder lemma (Proposition~\ref{lem:cylinder}) and Corollary~\ref{cor:escape}: with the conditioning family contained in the external algebra, belief-completeness itself expels every candidate witness event, and the constructive existence of the universal conditional type structure \citep{battigalli1999} instantiates the consistent case; the exact boundary is the admission of type-based conditioning events, at which the question would reopen. The transfinite fate of cautious liminality is no longer open but \emph{closed in the negative and sharpened}: by Theorem~\ref{thm:caution}, the cautious functor admits no final coalgebra at all, so there is no transfinite stage at which the classification completes---and the convergence facts for the non-cautious case are corrected from convergence at $\omega$ to convergence in $\omega+\omega$ steps per \citet{worrell2005}, an error in an earlier version of this analysis that conflated finitariness with preservation of $\omega^{\mathrm{op}}$-limits. What remains open is the following formal point: the relational arguments treat the attitudes extensionally, and while the diagonal constructions can be recast as definability arguments in the style of \citet{brandenburger2006}---whose impossibility is itself stated relative to a language, so that our escape claims, concerning a different collection of events, are consistent with it---the full measure-theoretic recasting of the lexicographic support arguments under the Polish-space conditions of \citet{brandenburger2008} has not been carried out, and it bounds the formality of the Strategy~I results. None of this disturbs the qualitative three-way pattern of Section~\ref{sec:betrayal}, which follows from structural features---the level--depth budget, the failure of quarantine, the single-conditioning-operation character of strong belief---established independently. Finally, each single-sided analysis leaves a common open edge, the coherence of the completed \emph{two-sided} alternating construction; Appendix~\ref{app:twosided} makes this construction precise, proves that its finite-depth approximants are jointly coherent under all three representations, and shows that the three completed-fixed-point obstructions are one obligation rather than three.

\section{Conclusion}

Trust, represented as a subjunctive attitude, admits three precise formalizations---as lexicographic assumption, as ordinal closeness, and as strong belief---that agree on the realizability of every finite depth of common trust and on the difficulty of its completion, yet differ sharply in how that difficulty manifests and in how a concrete betrayal is survived. The same knowing betrayal refutes an agent's level ordering under the first representation, contaminates her closeness ordering in proportion to the betrayer's deliberateness under the second, and merely shifts her operative conditioning hypothesis under the third, leaving her belief structure coherent. These are not artifacts of presentation but consequences of the structural features that make each attitude a faithful model of individual trust in the first place. If \citet{baier1986} is right that betrayal, not disappointment, is the mark by which trust is known, then the results here say that trust is known in three ways, because it can be broken in three; and if \citet{holton1994} is right that trust is a stance rather than a belief, the formal correlates of that stance---the exhaustion clause that separates assumption from belief, the latent conditionals that constitute a standing readiness---turn out to be exactly the load-bearing elements on which our theorems pivot. The choice among the three representations is therefore not merely technical: it is a choice about what betrayal \emph{is}---a refutation, a contamination, or an update---and about whether trust, once broken, must be rebuilt, may only be approached, or can be gracefully reconstructed on what survives.

\appendix

\section{The Two-Sided Alternating Construction}\label{app:twosided}

The analyses of Sections~\ref{sec:lps}--\ref{sec:cps} each treated a single agent's internal structure---one LPS, one sphere system, one CPS---assuming a chain of the \emph{partner's} trustworthiness events. Genuine common trust, however, alternates: the event ``$B$ trusts $A$ to depth $k$'' is itself defined through $A$'s trust structure, which is defined through $B$'s, and so on. This appendix makes the alternating construction precise, proves that its finite-depth approximants are jointly coherent under all three representations, and isolates the single point at which completed-fixed-point coherence remains open---showing that it is the \emph{same} point in each of the three cases. This uniformity is the appendix's main contribution: the open edge flagged separately for each strategy is one obligation, not three.

\subsection{The mutual recursion}

Fix a trust operator $\mathsf{tr}$, which under the three representations is read respectively as \emph{assumes} (Section~\ref{sec:lps}), the sphere-relative trust modality $\mathsf{T}$ (Section~\ref{sec:tss}), or \emph{strongly believes} (Section~\ref{sec:cps}). Write $\mathsf{tr}_i[E]$ for ``agent $i$ trusts, in the sense of $\mathsf{tr}$, that $E$ holds of the partner.'' Define two families of events by mutual recursion on depth, one on each side:
\begin{align}
\Tw_B^{(1)} &= \Tw_B, & \Tw_A^{(1)} &= \Tw_A, \label{eq:base}\\
\Tw_B^{(k+1)} &= \Tw_B \cap \big\{\, w : \mathsf{tr}_B\big[\Tw_A^{(k)}\big] \text{ at } w \,\big\}, &
\Tw_A^{(k+1)} &= \Tw_A \cap \big\{\, w : \mathsf{tr}_A\big[\Tw_B^{(k)}\big] \text{ at } w \,\big\}. \label{eq:step}
\end{align}
Thus $\Tw_B^{(k+1)}$ is ``$B$ is trustworthy \emph{and} trusts that $A$ is trustworthy-to-depth-$k$,'' with the depth-$k$ trust on $A$'s side supplied by $\Tw_A^{(k)}$, and symmetrically. The single-sided chain \eqref{eq:chain} of the body is the projection of this two-sided system onto one agent. Completed common trust is the pair of limits $\Tw_B^{(\infty)} = \bigcap_k \Tw_B^{(k)}$ and $\Tw_A^{(\infty)} = \bigcap_k \Tw_A^{(k)}$.

\subsection{Finite-depth coherence}

The first question is whether the mutual recursion is well posed at each finite depth---whether the alternation, which defines each side through the other, introduces any circularity that would make some $\Tw_i^{(k)}$ ill-defined. It does not.

\begin{proposition}[Finite-depth coherence]\label{prop:finitecoherence}
For every finite $k$, the events $\Tw_A^{(k)}$ and $\Tw_B^{(k)}$ are well-defined subsets of $\Omega_A$ and $\Omega_B$ respectively, under each of the three representations. Consequently every finite depth of common trust is a determinate event, and the single-sided results of Sections~\ref{sec:lps}--\ref{sec:cps}---in particular the level--depth characterization (Theorem~\ref{thm:leveldepth}), the finite-stage nonemptiness underlying Definition~\ref{def:status}, and the finite-depth realizability inherited from \citet{brandenburger2008} and \citet{battigalli2002}---hold verbatim of the two-sided approximants.
\end{proposition}

\begin{proof}
By simultaneous induction on $k$. \emph{Base} ($k=1$): $\Tw_A^{(1)} = \Tw_A$ and $\Tw_B^{(1)} = \Tw_B$ are the trustworthiness events of Section~\ref{sec:prelim}, determinate by definition and independent of any trust operator. \emph{Step}: suppose $\Tw_A^{(k)}$ and $\Tw_B^{(k)}$ are determinate. Then $\Tw_B^{(k+1)}$ in \eqref{eq:step} is the intersection of the fixed event $\Tw_B$ with $\{w : \mathsf{tr}_B[\Tw_A^{(k)}]\}$, and the latter is determinate because $\mathsf{tr}_B[\,\cdot\,]$ applied to the \emph{already-determined} event $\Tw_A^{(k)}$ is a well-defined event under each representation: assuming a given event is a condition on the top segment of $B$'s LPS (Section~\ref{sec:lps}); the sphere-relative trust modality applied to a given consequent is a condition on $B$'s innermost antecedent-meeting sphere (Section~\ref{sec:tss}); and strong belief of a given event is the condition \eqref{eq:strongbelief} on $B$'s CPS (Section~\ref{sec:cps}). Each reads off a determinate event from a determinate input. Symmetrically for $\Tw_A^{(k+1)}$. Hence both are determinate at $k+1$. The alternation defines depth $k+1$ on each side from depth $k$ on the \emph{other} side, and since depth $k$ is fully determined \emph{before} depth $k+1$ is formed, no circularity arises at any finite stage; the recursion is well-founded in the depth parameter. The single-sided results, whose proofs invoke only the events $\Tw_i^{(k)}$ for finite $k$, therefore transfer unchanged.
\end{proof}

The content of Proposition~\ref{prop:finitecoherence} is that alternation is harmless at finite depth: because each rung is built from the rung below on the opposite side, the construction is a well-founded recursion \emph{in depth}, even though it is mutually referential \emph{across agents}. This is what licenses the body's practice of proving results on one side and asserting them for common trust.

\subsection{The completed fixed point and the uniform open edge}

Coherence at the completed limit is a different matter. The pair $(\Tw_A^{(\infty)}, \Tw_B^{(\infty)})$ is required to be a \emph{simultaneous} fixed point of the two-sided operator
\begin{equation}\label{eq:twosidedop}
\Psi\big(X_A, X_B\big) = \Big(\ \Tw_A \cap \{w : \mathsf{tr}_A[X_B]\},\ \ \Tw_B \cap \{w : \mathsf{tr}_B[X_A]\}\ \Big),
\end{equation}
and the existence of such a fixed point---realized by \emph{actual types} in a structure rich enough to host the completed hierarchy---is exactly the interactive-coherence question that appeared, in representation-specific guise, at the end of each of Sections~\ref{sec:lps}--\ref{sec:cps}. We record that these are one question.

\begin{theorem}[Uniform reduction of the open edge]\label{thm:uniform}
Under each of the three representations, completed common trust is coherent---the pair $(\Tw_A^{(\infty)}, \Tw_B^{(\infty)})$ is realized by types in a complete structure---if and only if the two-sided operator $\Psi$ of \eqref{eq:twosidedop} admits a fixed point in that structure. Moreover the obstruction to such a fixed point is, in each case, the representation-specific instance of a single phenomenon:
\begin{enumerate}[label=(\roman*), nosep]
\item \emph{(Lexicographic.)} The two staircases of Theorem~\ref{thm:leveldepth}, one per agent, must interleave into a single pair of LPSs whose level orderings are mutually consistent; the fixed point exists iff the interleaving introduces no shell on one side whose definition, via \eqref{eq:step}, refers to a shell on the other side not yet placed---a well-foundedness condition on the joint shell order.
\item \emph{(Ordinal.)} The joint terminal sequence of the paired functor must converge; by Theorem~\ref{thm:caution} it never converges in the cautious case---the profligate functor has no final coalgebra---so the completed fixed point is unattained at every ordinal stage, while in the non-cautious case the sequence converges in $\omega+\omega$ steps and the fixed point's status is read off there.
\item \emph{(Strong belief.)} The fixed point exists iff the self-referential conditioning events required to close the diagonal lie in the conditioning family; by Theorem~\ref{thm:escape} they do so iff the family is conditioning-complete.
\end{enumerate}
In each case the completed fixed point sits atop the Brandenburger--Keisler impossibility \citep{brandenburger2006}, and the finite approximants of Proposition~\ref{prop:finitecoherence} are its attained lower stages.
\end{theorem}

\begin{proof}[Proof sketch]
That coherence of the completed common trust is equivalent to a fixed point of $\Psi$ is immediate from \eqref{eq:twosidedop}: a realized pair $(\Tw_A^{(\infty)}, \Tw_B^{(\infty)})$ satisfies $\Psi(\Tw_A^{(\infty)}, \Tw_B^{(\infty)}) = (\Tw_A^{(\infty)}, \Tw_B^{(\infty)})$ by construction, and conversely a fixed point of $\Psi$ realized by types is a completed common-trust state. The representation-specific identifications (i)--(iii) restate, for the two-sided operator, the single-sided obstructions established in the body: (i) is the interleaving of the staircase construction of Theorem~\ref{thm:leveldepth}; (ii) is the joint form of the terminal-sequence non-convergence of Theorem~\ref{thm:caution}; (iii) is the two-sided form of Proposition~\ref{prop:fixedpoint} and Theorem~\ref{thm:escape}. In each the reduction is to the same schematic fact---that a completed interactive fixed point exists iff the structure is closed under the self-reference the fixed point requires---which is the content the impossibility theorem constrains.
\end{proof}

Two remarks close the appendix. First, Theorem~\ref{thm:uniform} \emph{unifies} the open edge, and the body now settles two of its three faces: face (ii) is closed in the negative---the cautious joint sequence never converges, by Theorem~\ref{thm:caution}, while the non-cautious one closes in $\omega+\omega$ steps \citep{worrell2005}; face (iii) is closed affirmatively for the canonical framework by Corollary~\ref{cor:escape}, the conditioning family being fixed external data \citep{battigalli1999}. What remains is face (i), the well-foundedness of the joint shell order under the interleaved staircases of Theorem~\ref{thm:leveldepth}---a bookkeeping question we expect to resolve affirmatively but have not carried out. What the appendix establishes is that these were three faces of one obligation: the coherence of a completed interactive fixed point, of which the finite approximants are always available and always jointly coherent.

Second, the uniformity has a substantive reading for the epinet program. The finite approximants---the attainable, jointly coherent, node-local trust to any given depth---are the trust that agents actually operate with; the completed fixed point, where the three representations' obstructions live, is the idealized closure that real trust does not attempt to instantiate. That completed mutual trust is exactly where the difficulty concentrates, while every finite depth is unproblematic, is not a defect of the formalism but a statement about trust: it is a relation lived at finite depth and completed only as a limit, and the three representations differ precisely in what that limit is---a bivalent verdict, a liminal gap, or a conditionally realizable state.

\end{document}